\documentclass{article}
\usepackage[short]{optidef}
\usepackage{amsmath}
\usepackage{mathtools}
\usepackage{amssymb}
\usepackage{amsthm}
\usepackage{amsmath,amsfonts,bm}
\newtheorem{prop}{Proposition}
\newtheorem{theorem}{Theorem}
\theoremstyle{definition}

\newcommand{\norm}[1]{\left\lVert#1\right\rVert}

\def\R{{\mathcal{R}}}

\def\C{{\mathcal{C}}}
\def\Cdyn{{\mathcal{C}^\text{dyn}}}

\def\dm{{\mathrm{d}}}

\def\sA{{\mathcal{A}}}
\def\sB{{\mathcal{B}}}

\def\bdot{{\boldsymbol{\cdot}}}
\def\path{{\phi}}

\def\am{{\mathrm{\alpha}}}
\def\bm{{\mathrm{\beta}}}

\PassOptionsToPackage{numbers, compress}{natbib}

 \usepackage[final]{neurips_2019}

\usepackage[utf8]{inputenc} 
\usepackage[T1]{fontenc}    
\usepackage{hyperref}       
\usepackage{url}            
\usepackage{booktabs}       
\usepackage{amsfonts}       
\usepackage{nicefrac}       
\usepackage{microtype}      

\usepackage{algorithm}
\usepackage{algorithmic}
\title{Optimal Unsupervised Domain Translation}

\newcommand*\samethanks[1][\value{footnote}]{\footnotemark[#1]}

\author{%
  Emmanuel de Bézenac\thanks{Equal contribution}\\ 
  Sorbonne Université, UMR 7606, LIP6, Paris\\
  \texttt{emmanuel.de-bezenac@lip6.fr}\\
  \And Ibrahim Ayed\samethanks\\
  Sorbonne Université, UMR 7606, LIP6, Paris\\
  Theresis Lab, Thales, Paris\\
  \texttt{ibrahim.ayed@lip6.fr} \\
  \And Patrick Gallinari\\
  Sorbonne Université, UMR 7606, LIP6\\
  Criteo AI Lab, Paris\\
  \texttt{patrick.gallinari@lip6.fr} \\
}

\begin{document}

\maketitle

\begin{abstract}
Domain Translation is the problem of finding a meaningful correspondence between two domains. Since in a majority of settings paired supervision is not available, much work focuses on Unsupervised Domain Translation (UDT) where data samples from each domain are unpaired. Following the seminal work of \textit{CycleGAN} for UDT, many variants and extensions of this model have been proposed. However, there is still little theoretical understanding behind their success. We observe that these methods yield solutions which are approximately minimal w.r.t. a given transportation cost, leading us to reformulate the problem in the Optimal Transport (OT) framework. This viewpoint gives us a new perspective on Unsupervised Domain Translation and allows us to prove the existence and uniqueness of the retrieved mapping, given a large family of transport costs. We then propose a novel framework to efficiently compute optimal mappings in a dynamical setting. We show that it generalizes previous methods and enables a more explicit control over the computed optimal mapping. It also provides smooth interpolations between the two domains. Experiments on toy and real world datasets illustrate the behavior of our method.\\

\end{abstract}


Given pairs of elements from two different domains, \textit{domain translation} consists in learning a mapping from one domain to another, linking these paired elements together. If we consider pairs of photographs of a given scene, and associated artistic paintings of the same scene, our mapping would learn to map the photographs to the associated paintings, and conversely, to map paintings to the associated photographs. A wide range of problems can be formulated as translation, including image-to-image~\cite{pix2pix2016} or video-to-video~\cite{vid2vid_paired} translation, image captioning~\cite{stackgan}, natural language translation~\cite{trad}, etc. In the general case, obtaining paired examples can be hard. The \textit{unpaired} setting where samples from both domains are available without any pairing allows us to tackle a wider range of problems. A seminal work in this direction has been the CycleGAN model proposed by \cite{CycleGAN} which has led to extensions for many problems. In the following, we adopt this general framework for Unsupervised Domain Translation (UDT).

Despite their impressive successes, there remains little theoretical understanding on why these models work. \cite{galanti, esther} have observed that the approach in \cite{CycleGAN} is ill-conditioned: in most cases, any pairing between samples of both domains minimizes the loss. This is in contradiction with empirical evidence and shows that there must be an implicit bias towards well-behaved mappings. \cite{galanti} made a first step in this direction by demonstrating that \textit{semantically} coherent mappings are obtained by networks of minimal relative \textit{complexity}, a notion related to the number of hidden layers in a neural network (NN).

While this idea is interesting, functional complexity alone cannot account for the semantical constraints of a given UDT problem. The aim of this work is to characterize rigorously the properties of successful models by introducing an abstract cost which constrains the problem and leads to a unique solution for the chosen cost, while allowing to implement a meaningful mapping between the two domains. Our investigations are supported both by empirical observations and theoretical analysis: we show that the added cost transforms UDT into an Optimal Transport problem which we then formulate in a dynamical setting. This provides links with neural network implementations and leads us to a natural, easily implementable general algorithm and leads us to a natural, easily implementable general algorithm. It also allows for a more precise control over the underlying semantics.

Our main contributions are the following:
\begin{itemize}
  \setlength\itemsep{0em}
    \item We show that current methods for solving UDT are ill-conditioned and we propose an explicit additional cost that will allow us to formalize a well posed UDT instance.
    \item We show that this new formulation is equivalent to finding the Optimal Transport (OT) between the domains \textit{w.r.t.} to this cost. This provides a new framework for a theoretically grounded UDT with the existence and unicity of the learned mapping for a wide range of cost functions.
    \item Based on the dynamical formulation of OT, we introduce  a new method to retrieve the optimal transport plan, \textit{Dynamic Optimal Neural Translation} (DONT) that leads to efficient implementations.
    \item We investigate the practical implications of the theoretical results and we illustrate experimentally how it is possible to learn accurate mappings with interpolation capacities.

\end{itemize}

\section{Unsupervised Domain Translation}
\label{sec:udt}

We introduce below two constraints characterizing a large family of UDT methods inspired by \cite{CycleGAN} and subsequent works. We then show that this formulation corresponds to an ill posed problem, and we illustrate using a toy example how these methods have an inductive bias towards learning minimal mappings. We formalize this intuition by defining  a new instance of UDT as an optimal transport problem. This guaranties well posedness: the solution to this problem exists and is unique.

\subsection{Introduction}
\label{sec:intro_udt}

Adopting a probabilistic viewpoint, both domains can be expressed as probability distributions $\am$ and $\bm$ supported over compact sets $\sA,\sB \subseteq \R^d$.  UDT consists in finding mappings ${T\colon \sA \to \sB}$ and ${S\colon \sB \to \sA}$, such that $(T,S)$ yields \textit{semantically meaningful} pairings.
The latter notion remains informal in the literature and is enforced by carefully tuning architectures and parameters. 

CycleGAN \cite{CycleGAN} and its extensions ~\cite{unsup_trans, sr_gan, speech2text, stargan} base their approach on the two following constraints:
\begin{itemize}
\label{translation_cond}
    \item[--] \textbf{\textit{Coherence}\footnote{The \textit{push-forward} $f_\sharp \rho$ is defined as ${f_\sharp \rho (B) = \rho(f^{-1}(B))}$, for any measurable set $B$. Said otherwise, coherence means that $T$ maps $\alpha$ to $\bm$ and $S$ does the reverse.} }: $\quad T_\sharp \am = \bm$, $\quad$ and $\quad S_\sharp \bm = \am$. 
    
    \item[--] \textbf{\textit{Inversibility}\footnote{Notation $f \overset{\mu-a.s.}{=} g$ expresses that $\int_B f \,\dm\mu = \int_B g\, \dm \mu$, for any measurable set $B$. Intuitively, this means that the equality is true over all sets of non-null $\mu$ measure.}}:  $\quad S\circ T \overset{\am-a.s.}{=} \text{id}$, $\quad$ and $\quad T\circ S\overset{\bm-a.s.}{=} \text{id}$,

\end{itemize}

\textit{Coherence} ensures that the transformations exactly map the two domains one onto the other. Usually, it is enforced by using a loss term of the form $D(T_\sharp \am, \bm) + D(S_\sharp \bm, \am)$, where $D$ corresponds to some measure of discrepancy between probability distributions. $D$ may be implemented using generative adversarial training~\cite{CycleGAN}, a denoising auto-encoder loss~\cite{unsup_trans, speech2text}, a distance in a shared projection space~\cite{bridging}, etc. \textit{Inversibility} was proposed as a way to construct information preserving transformations and is usually enforced via a \textit{cycle-consistent loss}\footnote{The $p-$norm of $f$ against measure $\mu$ is defined as $\norm{f}_{\mathcal{L}^p(\mu)} = (\int_{\R^d} |f|^p \dm \mu)^{\frac{1}{p}}$}: $\norm{S \circ T - \text{id}}_{L^p(\am)} + \norm{T \circ S - \text{id}}_{L^p(\bm)}$. For example, in \textit{CycleGan}~\cite{CycleGAN}  $D$ corresponds to an adversarial loss, and a cycle-consistent loss with the $L^1$ norm is used.

Neither invertibility nor coherence imply that the semantic information for a specific element of the input domain is conserved by the mapping. When $(\am,\bm)$ have a discrete support, any pairing is both invertible and coherent. In the continuous case, there is an infinite number of mappings verifying the constraints~(see section \ref{supp:wellpos} of the appendix for a proof): For example, a transformation associating beach photographs with mountain paintings and mountain photographs with beach paintings would not violate the constraints, while not being semantically meaningful.

\subsection{Current methods are Biased towards Small Transformations}

\begin{figure}
    \includegraphics[width=1.\textwidth]{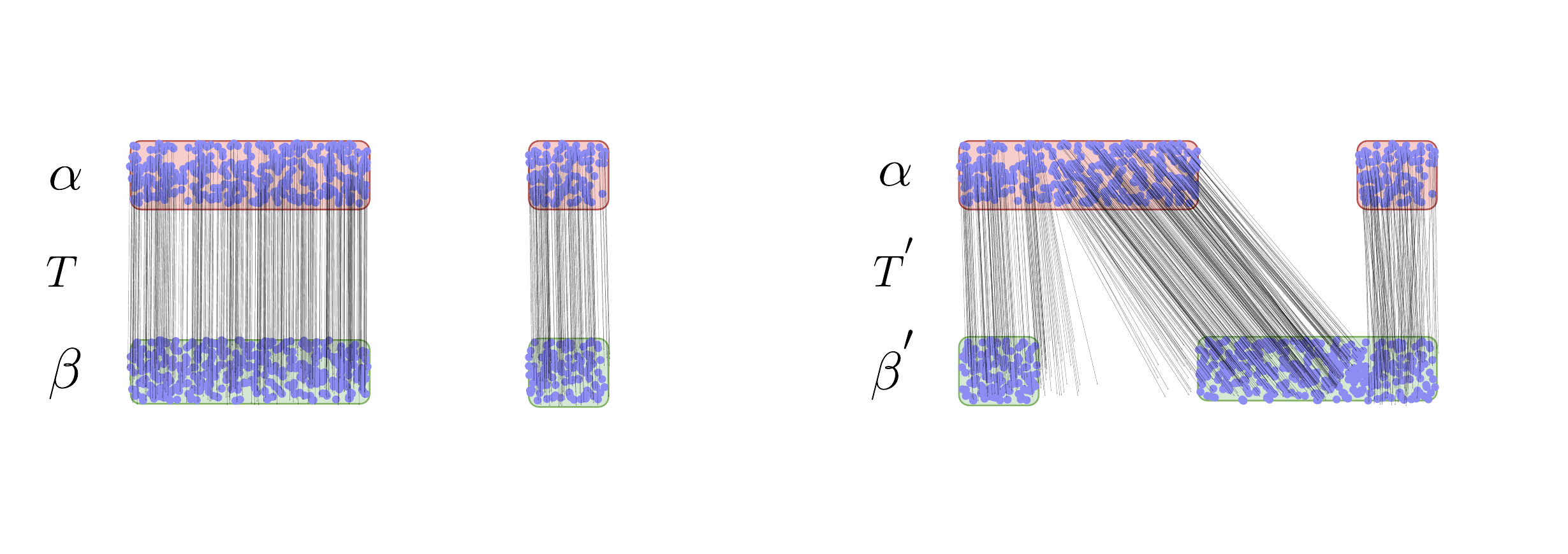}
    \vspace{-1.5cm}
    \caption{Two UDT tasks: map the 2D $\alpha$ distribution onto the 2D $\beta$ distribution. On the left $\beta$ is a translated version of $\alpha$ while on the right it is translated and rotated. Lines illustrate a typical mapping learned by CycleGAN.}
    \label{fig:toy}
\end{figure}

Before defining UDT as an OT problem, let us intuitively motivate this choice using a simple example of the typical CycleGAN behavior. Fig. \ref{fig:toy} shows two UDT instances. For the first one, data generated from a 2D distribution $\alpha$ are translated to produce distribution $\beta$ (left); for the second one, the same $\alpha$ is translated and rotated (right) to produce $\beta'$. For both, UDT consists in mapping the alphas to the betas. The two  mappings learned by CycleGAN, respectively $T$ and $T'$, are materialised by lines joining $(x,T(x))$ (left) and $(x,T'(x))$ (right) \footnote{Note that the mappings are not perfectly learned for this example, so that some lines have no end point in the target domain}. While the left situation is satisfying, the one on the right is not. This illustrates two points: (1) CycleGan seems to have an inductive bias towards minimal mappings that do not transform much the data and (2) this  is satisfying when the distributions are close one to the other (Fig. \ref{fig:toy} left), but it does not allow capturing more complex semantics (right). Thus, information about the task must be encoded in the problem. We will formalize these points by defining transformations with minimal transport cost, allowing us to select appropriate cost functions for a specific problem, and potentially tackle new UDT problems.

\subsection{Unpaired Domain Translation as Optimal Transport}


In the following, we consider that we have access to a finite number of samples of both distributions $\am$, $\bm$ which are assumed to be absolutely continuous \textit{w.r.t.} to the Lebesgue measure. Let us consider an abstract cost function $c(x, y)$, low when $x$ and $y$ are semantically similar, and high otherwise. In practice (Section \ref{sec:expes}), $x$ and $y$ will be learned representations of the data, typically NN projections, and will encode semantic information. In this case, $c$ could correspond to a distance between $x$ and $y$. The mapping of $x$ through $f\colon \R^d \to \R^d$ then incurs a cost $c\!\left(x, f(x)\right)$. We define a \textit{semantic preserving mapping} as a mapping with a low cost for all elements on the support of distribution $\mu$ of $x$, \textit{i.e.} a low value for $\int_{\R^d} c\!\left(x, f(x)\right) \, \dm \mu(x)$. In order to solve an UDT problem, one must then select the right cost for a given problem, and look for the mapping minimizing this cost.

Out of all mappings satisfying the \textit{invertibility} and \textit{coherence} conditions, we are now looking for mappings $T$ and $S$ that are of minimal cost \textit{w.r.t.} $c$:
\begin{mini}
    {T, S}{\int_{\R^d} c\!\left(x, T(x)\right) \, \dm \am(x) \; + \; \int_{\R^d} c\!\left(y, S(y)\right) \, \dm \bm(y)}
    {}{}
    \addConstraint{\quad T_\sharp \am = \bm, \quad S_\sharp \bm = \am}{}   
    \addConstraint{S\circ T \overset{\am-a.s.}{=}\text{id}, \quad T\circ S\overset{\bm-a.s.}{=} \text{id}}{}
\label{eq:ot}
\end{mini}
This optimisation problem can be rewritten in a simpler form, which is exactly the classical Monge formulation of optimal transport (OT).
    \begin{mini}
    {T}{ \C(T) = \int_{\R^d} c(x, T(x)) \, \dm \am(x)}
    {}{}
    \addConstraint{T_\sharp \am }{=\bm}
    \label{eq:monge}
    \end{mini}
    
The demonstration directly stems from the following powerful result \cite{Santambrogio} which not only shows that UDT is indeed an OT problem, but also provides existence and unicity of the optimal mappings $T$ and $S$, for a wide range of cost functions $c$.

\begin{theorem}
    Let $\am$, $\bm$ be absolute continuous measures. If $c(x,y)=h(x-y)$ where $h$ is strictly convex, then there \textbf{exists a unique} couple $(T,S)$ of transformations such that:
    \begin{itemize}
        \item $T_\sharp\mathcal{\am}=\mathcal{\bm}$, and $S_\sharp\mathcal{\bm}=\mathcal{\am}$,
        \item $\C(T)$ is minimal, and $S$ is the minimal transport from $\bm$ to $\am$.
        \item $T\circ S \overset{\mathcal{\bm}-a.s.}{=} \text{id}$, and $S\circ T \overset{\mathcal{\am}-a.s.}{=} \text{id}$,
    \end{itemize}
\end{theorem}

This result has several interesting implications:
\begin{itemize}
    \item[--] Existence and unicity of the minimal transformation are guaranteed in theory. A possible explanation of why meaningful translation are achieved by simple UDT models is that they optimize an implicit transport cost. Empirical evidence supporting this idea is provided in appendix \ref{supp:bias}.
    \item[--] When solving \eqref{eq:monge} all the constraints present in \eqref{eq:ot} are automatically verified, without having to enforce them explicitly, apart from coherence over $T$.
    \item[--] Invertibility is naturally guaranteed. Solving on both $T$ and $S$ is not necessary, as solving only on $T$ yields the same result\footnote{This obviously only applies from a theoretical point of view when the optimum can be practically calculated so that optimizing over $S$ can still be useful in some practical situations.}, although in practice we still have to find $S$.
\end{itemize}



\section{\label{sec:dont}Dynamical Optimal Neural Transport (DONT)}
This formulation introduces a relation between UDT as OT, but does not provide a link with the NN models used for UDT and CycleGAN inspired methods. Considering now the dynamical point of view of OT will allow us to develop this link. It will provide us with a simple way to enforce OT for costs of the form $c(x, y) = \|x-y \|^p_p, p > 1$ and for developing an efficient algorithm that can perform inference on new data points.


\subsection{\label{subsec:resnets}Residual Networks and Dynamical Systems}

Residual networks have shown to be a core component for solving UDT \cite{CycleGAN}. They apply transformations $x_{n+1}=x_n + v_n(x_n)$, in an iterative fashion which corresponds to the discretization of an Ordinary Differential Equation (ODE) ~\cite{he_deep_2016}. In the continuous limit, a ResNet implements an ODE\cite{weinan, NODE}, characterized by a vector field $v$:
\begin{equation}
\label{eq:state}
    \frac{dx_t}{dt} = v_t(x_t)
\end{equation}
This equation defines a latent trajectory gradually moving an input $x_0$ from $t=0$ to $1$ along $v_t$, producing an output $x_1$. Intuitively, if at each step $t$, the velocity $v_t$ applied to $x_t$ is small, then the overall transport of $x_0$ should be small as well. It turns out that this intuitive argument can be made rigorous through the dynamical formulation of OT.

\subsection{\label{subsec:dyn_formulation}A Dynamical Formulation of the Optimal Transport problem}

\begin{figure}[htb!]
    \centering
    \includegraphics[width=.9\textwidth]{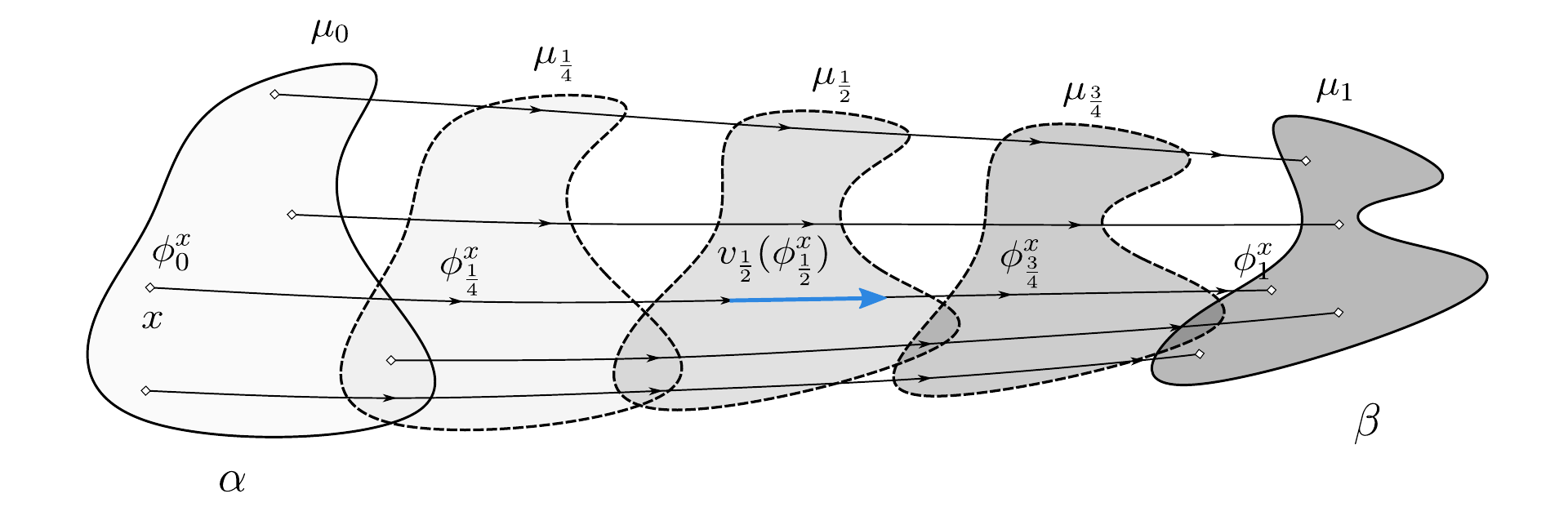}
    \caption{ The Figure illustrates successive steps of the dynamic transportation of $\am$ to $\bm$ together with the notations used in the text. Each step could for example correspond to a transformation performed by an elementary module of a ResNet.}
    \label{fig:dynamic}
\end{figure}






In the dynamical formulation of OT \cite{Santambrogio}, probability mass of $\am$ is not directly pushed to $\bm$ by the static mapping $T$ of Section (\ref{sec:intro_udt}), but is gradually transported in time along a given path of minimal cost. This path is given by the velocity vector field $v$ solution of the problem:
\begin{mini}
    {v}{ \Cdyn(v) = \int_0^1 \norm{v_t}^p_{L^p(\mu_t)}\dm t}
    {}{}
    \addConstraint{\partial_t \mu_t + \nabla \cdot (\mu_t v_t)}{= 0, \mu_{0}= \am, \mu_{1}= \bm}
    \label{eq:dyn}
\end{mini}

where $(\mu_t)_{t \in [0, 1]}$ is the minimal energy path from $\am$ to $\bm$ in a measure space (see Fig. \ref{fig:dynamic}). For costs of the form $c(x, y) = \| x -y \|^p_p$, this formulation is equivalent to (\ref{eq:monge}), meaning that the overall transformation from $\am$ to $\bm$ is the same, and thus yields the same transport cost. Directly solving (\ref{eq:dyn}) requires solving the continuity equation ${\partial_t \mu_t + \nabla \cdot (\mu_t v_t)=0}$, starting from the initial density of $\am$. However $\am$ is unknown: we only have ever have access to samples and estimating $\am$ in high dimensional spaces is prohibitive.
Instead of modeling the evolution of the density in time, we model the trajectories induced by the elements $x\in\sA$ the support of $\am$, displaced along the vector field $v$. Let $\path\colon \sA \times [0, 1] \to \R^d\!, \: (x, t) \mapsto \path^x_t$ describe the position of elements $x$ of $\am$ at time $t$, when they are displaced along $v$ (see Fig. \ref{fig:dynamic}). Then the optimization problem can be equivalently written as:
\begin{mini}
    {v}{ \Cdyn(v) = \int_0^1 \|v_t\|_{L^p((\path^\bdot_t)_\sharp\am)}^p\, \dm t}
    {}{}
    \addConstraint{\partial_t\path_t^x}{ = v_t(\path_t^x)}
    \addConstraint{\path^\bdot_0}{= \text{id}}
    \addConstraint{(\path^\bdot_1)_\sharp\mathcal{\am} }{=\mathcal{\bm}}
    \label{eq:lagrangian_form}
\end{mini}

where function $\phi_t^{\cdot}:\sA \to \R^d\!$ is the transport map at time $t$. Note that by replacing $x_t$ by $\path_t^x$ in the forward equation $\partial_t\path_t^x = v_t(\path_t^x)$, we recover the ResNet continuous time limit. This will allow us to link the latent trajectories of residual networks with these minimal length trajectories.

\subsection{Instantiation}
\label{sec:algorithm}
This section introduces how to solve equation (\ref{eq:lagrangian_form}), and thus how to develop a practical UDT algorithm.

\paragraph{Discretization.} We start by discretizing the forward equation $\partial_t\path_t^x = v_t(\path_t^x)$ in time via a $K$ step Euler discretization, starting from $\path^x_0=x$:
\[
\forall x,\ \path^x_{(k+1)\Delta t} = \path^x_{k\Delta t} + \Delta t \, v_{k\Delta t}(\path^x_{k\Delta t})
\]
Here $K$ is the total number of discret steps defining the transformation. Note that by doing so, a residual network architecture is recovered.
For readability purposes, in the following, $\Delta t$ will be omitted in $\path^x_{j\Delta t}$ and $v_{j\Delta t}$. We now replace the unknown distribution $\am$ with its empirical $N$ samples counterpart $\frac{1}{N}\sum_{x\in\text{Data}_\am}\delta_x$, with ${\text{Data}_\am}$, samples from $\alpha$, we obtain:
\[
(\path^\cdot_{k})_\sharp\am\approx \frac{1}{N}\sum_{x\in\text{Data}_\am}\delta_{\path^x_{k}}
\]
corresponding to the empirical distribution induced by the displacement incurred up to step $k$. The cost can now be estimated using the data, summing up the lengths of the trajectories induced by the input data: 
\[
\Cdyn(v)  \approx \frac{\Delta t}{N} \sum^K_{k=1}\sum_{x\in\text{Data}_\am} \norm{v_{k}(\path^x_{k})}_p^p
\]

Vector field $v_{k}$ is a function transforming $\path^x_{k}$ into $\path^x_{k+1}$. We parameterize it  for each step $k$ using a neural network of parameters $\theta_k$, denoted $v^{\theta_k}$. The optimization problem then amounts at minimizing the norm of residuals, under constraints:
\begin{mini}
    {\theta}{ \mathcal{C}_d(\theta) = \sum^K_{k=1}\sum_{x\in\text{Data}_\am} \norm{v^{\theta_k}\!(\path^x_{k})}_p^p}{}{}
    \addConstraint{\forall x,\ \path^x_{k+1}}{ = \path^x_{k} + \Delta t \, v^{\theta_k} \!(\path^x_{k})}    \addConstraint{\qquad \qquad \path^\cdot_0=\text{id},}{\ \ (\path^\bdot_1)_\sharp\mathcal{\am} }{=\mathcal{\bm}}
    \label{eq:discrete_opt}
\end{mini}

\paragraph{Enforcing Boundary Conditions.} The forward equation $\path^x_{k+1} = \path^x_{k} + \Delta t \, v^{\theta_k} \!(\path^x_{k})$ is trivially verified, as is $\path^\cdot_0=\text{id}$. This is not the case for the coherence constraint $(\path^\bdot_1)_\sharp\mathcal{\am} =\mathcal{\bm}$ ensuring that input domain $\am$ maps to the target domain $\bm$. In order to implement a numerical algorithm, we optimize the  Lagrangian associated to (\ref{eq:discrete_opt}) (the trivially enforced constraints are not made explicit here), introducing a measure of discrepancy $D$ between output and target domains:
\begin{mini}
    {\theta}{ \mathcal{C}_d(\theta) + \frac{1}{\lambda_i}D\!\left((\path^\bdot_1)_\sharp\am, \, \bm\right)}{}{}
    \label{eq:discrete_unconstrained_opt}
\end{mini}

where the sequence of Lagrange multipliers $(\lambda_i)_i$ converges linearly to $0$ during optimization, ensuring the constraint is met. This optimization problem is solved using stochastic gradient based techniques. As in most approaches for UDT, $D$ may be implemented using generative adversarial networks \cite{fgan}, or any other appropriate measure of discrepancy between samples, such as kernel distances \cite{gretton}, or OT based distances \cite{cuturi}.


\paragraph{Inverse Mapping.} Training is done \textbf{only for the forward equation}. To obtain the inverse mapping after training, the forward equation is solved in reverse mode. This is immediate and simply amounts at iterating ${y_{k-1} = y_{k} - \Delta t \, v^{\theta_k}(y_{k})}$, starting from a sample $y_K$ from $\bm$.

\begin{figure}
    \centering
\includegraphics[width=1.\textwidth]{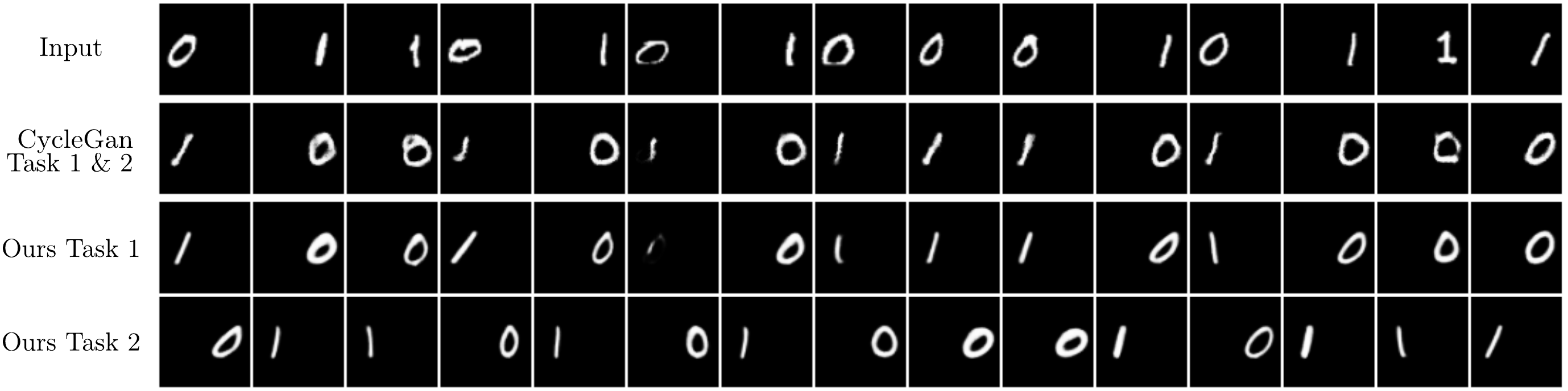}
    \caption{MNIST Digit Swap Task: first row - input, 2nd row - transformation learned by CycleGAN for both tasks, 3rd and 4th row - transformations learned by our model for each task (see text for task definition).}
    \label{fig:mnist}
\end{figure}

\begin{figure}
    \centering
\includegraphics[width=1\textwidth]{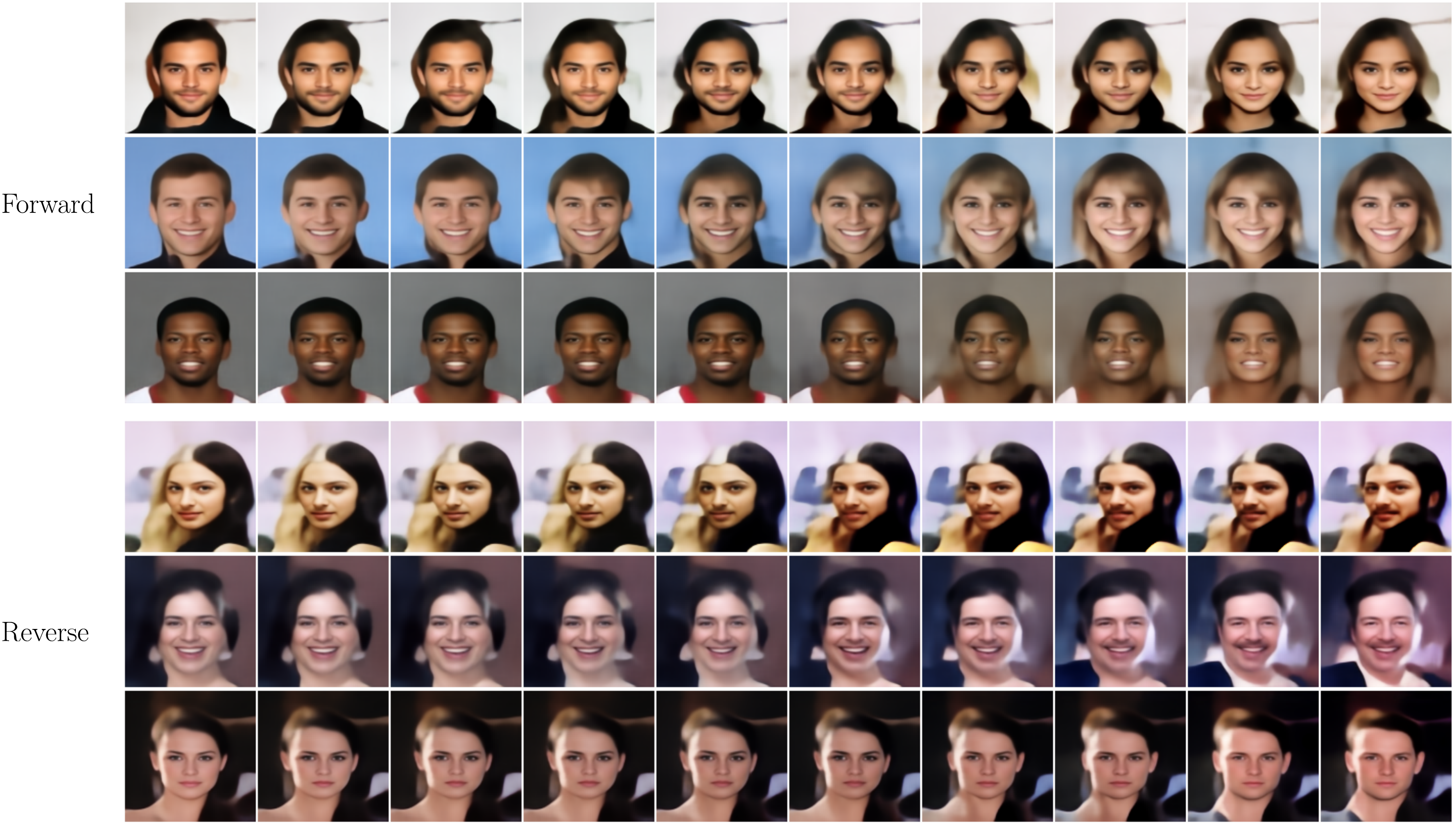}
    \caption{Male to Female translation (top) and the inverse (bottom). Intermediate images correspond to the interpolations provided by the network's intermediate layers. The reverse mapping is obtained by simply inverting the forward network once trained (see text for explanation).
    }
    \label{fig:celeba}
\end{figure}

\paragraph{Transporting in Neural Network Semantic Space.} The RGB image space is not always ideal to work with. For instance, the Resnet architecture used for many image-to-image translation tasks~\cite{pix2pix2016, CycleGAN} encodes the images in latent space, lowering spatial resolution before transporting the image through the residual blocks, and decoding back into RGB space. Instead of directly transporting one domain to another, we first \textbf{embed domains} in a more amenable space, using a pretrained \textbf{encoder-decoder pair}, to make encoding and decoding steps decoupled from the transport. In this \textit{semantic} space, transportation costs considered in Section \ref{sec:dont}, of the form $c(x, y) = \| x -y\|_p^p$ become meaningful: we wish to map one domain onto another, preserving as much semantics as possible. 

To summarize the previously introduced steps, a practical algorithm is proposed  in the supplementary material \ref{supp:details}. 

\section{\label{sec:expes}Experiments}
We illustrate now some interesting properties of our approach on two different settings. Additional exploratory analysis of CycleGan is available in the supplementary material \ref{supp:bias}.

\subsection{MNIST Digit Swap Task}
\label{sec:mnist}
This toy task making use of MNIST data, illustrates some limitations of the CycleGAN family of models and the benefits of using our OT formulation.

\begin{table}[ht]
\centering
\caption{Test-set results on the MNIST Digit Swap Task. Coherence Score: $\%$ of success for \textit{"$\am$ is mapped on $\bm$"}, Semantic Score Task 1 \& 2: $\%$ of success for \textit{"the mapped digit is both correct and at the right position"}.}
\begin{tabular}[t]{l|ccc}
\label{table:mnist}
& Coherence Score & Semantic Score Task 1 & Semantic Score Task 2 \\
\hline
CycleGan& $\mathbf{99.9\%}$ & $\mathbf{99.9\%}$ & $0.00\%$ \\
Ours (Task 1)& $\mathbf{99.9\%}$  & $\mathbf{99.9\%}$ & $0.0\%$\\
Ours (Task 2)& $\mathbf{99.0\%}$  & $15.42.\%$ & $\mathbf{83.4\%}$\\
\end{tabular}
\end{table}

Let us consider here two domains on MNIST digits: the first domain corresponds to $0$s placed on the left of the image and $1$s placed on the right (Fig. \ref{fig:mnist} row 1). For the second, the $0$s are placed on the right and $1$s on the left. In order for the coherence constraint to be verified (Sec. \ref{sec:intro_udt}), we can either swap the class of the digit and keep the position, or swap the position and leave the class unchanged. For the same set of domains, two different tasks are considered:
\begin{itemize}
  \setlength\itemsep{-0.1em}
    \item[--] \textbf{Task 1:} Associate a digit to a digit of the opposite class, keeping the position unchanged.
    \item[--] \textbf{Task 2:} Associate a digit to the same digit on the opposite position,
\end{itemize}

CycleGan trained on both tasks 1 and 2, consistently provides the same solution for both tasks (Fig. \ref{fig:mnist}, row 2). It actually solves Task 1 reasonably well but never solves Task 2. Our method solves task 1, simply using the transport cost $c(x,  y) = \| x -y \|_2^2$ (Fig. \ref{fig:mnist}, row 3). It is also able to solve Task 2 reasonably well (Fig. \ref{fig:mnist}, row 4).
To solve Task 2, we choose an appropriate cost function $c$, and look for the mapping that minimizes the transport \textit{w.r.t.} $c$. For $c$, we first find the components of the latent representations of the images that are the most correlated with the position. For that we a sparse linear classifier to distinguish the digits position, select the features with non 0 weights. We then \textit{turn off} the contribution of these features in the cost function. More specifically, we use $c(x, y) = \sum_i c_i |x_i-y_i|^2$, where $c_i = 0$ if the classifier's weight associated to component $x^i$ is non-zero, and $c_i = 1$ otherwise. We  use a \textit{Coherence Score} (discrepancy between output domain and target), and a \textit{Semantic Score} (how well we solve the given task) for each Task. These scores are obtained from a combination of predictions of two classifiers, trained to predict position, and class. Globally, coherence is achieved by both approaches, while semantic preservation is achieved on both tasks by our approach, and only on task 1 by CycleGAN (see Table \ref{table:mnist} and Fig. \ref{fig:mnist}). For implementations details, please refer to supplementary material Sec. \ref{supp:details}.

\subsection{CelebA Male to Female}
Fig. \ref{fig:celeba} illustrates how the model work for Male to Female translation (forward) and back (reverse) on the CelebA dataset. The images are first encoded using a pretrained encoder-decoder pair onto a $128$ dimensional space. This defines the semantic space on which an OT cost will be used. Using the proposed framework, the optimal forward mapping \textit{w.r.t.} cost $c(x, y) = \| x - y\|^2$ is learned -- without any cycle-consistency constraint. The inverse mapping is recovered after training, using the reverse equation introduced in Section \ref{sec:algorithm}. By decoding the intermediate layers of the neural network using the learned decoder, we obtain smooth image interpolations (columns 2 to 9). The model is able to perform gradually semantically meaningful transformations of the inputs, modifying the clothing, and the facial hair, and high-level attributes. Moreover, the reverse mapping works without having been trained. For additional samples and implementation details, please refer to the supplementary material \ref{supp:details}.


\section{Discussion}


Besides providing a theoretically founded framework for UDT, OT and dynamic OT formulations together with empirical evidence bring some intuition on why current implementations of UDT with CycleGAN related methods work well. All these method seem to minimize implicitely a transport cost as illustrated on the MNIST example above. This method is extremely simple to instantiate on ResNet like architectures. Note that although we have studied here the case of CycleGan and residual networks, the framework introduced in Sec. \ref{sec:udt} is more general. Other architectures like \textit{UNet} \cite{unet} making use of skip connections are also biased towards producing outputs  close to the inputs and can be used as well. This formalism can be used for solving any transport problem using a NN model.



The dynamical framework DONT presented in Section \ref{sec:dont} is not limited to UDT. For example, it could be used to analyze and constrain the dynamics occurring through the different layers of a neural network, in the spirit of \cite{Chen2018}. Note that more general costs than those considered in Section \ref{sec:dont} may be used, with provable guarantees (refer to \cite{figalli}).


Finally, leveraging theoretical and practical advances in OT on UDT and related problems may provide new ways to attack these problems, \textit{e.g.} \textit{Multi-Domain Translation} \cite{multi} is very similar to the problem of \textit{Multi-Marginal} OT (see \cite{cot}), and Many-to-Many Mappings \cite{augmented} can be linked with Entropic OT (\cite{cot}).

\section{Related Work}


Explaining why practical UDT methods work well when the problem is ill posed has motivated some recent worK. \cite{galanti, estimating}, show that coherent mappings are obtained by NNs of minimal \textit{functional complexity}, a notion related to the number of layers. They also prove that the number of such solution mappings is small. However, this notion of complexity does not account for problem dependent semantics. We achieve this by explicitly minimizing a task-dependant cost on a representation space where this cost makes sense. Using our framework, it is then possible to prove the existence and unicity of the solution for a wide range of costs. \cite{galanti, one_sided} show that learning a one-sided mapping is possible, but their formulation does not provide the inverse mapping like ours.

The dynamical approach presented in Sec. \ref{sec:dont} bears similarities with recent work \cite{NODE, ffjord}  making use of the link between residual networks and ODEs \cite{weinan}. Their objective is the design of new neural models and not solving a task like we do. Also different from us, they grid the space, and displace a \textit{known} density (usually Gaussian) along a learned trajectory. This is not possible in our case.

In the domain adaptation field, using Optimal Transport to help a classifier extrapolate has been around for some years, e.g. \cite{courty_OT_domainAdapt, courty2_OT_domainAdapt} use a transport cost to align two distributions. The task, although related is clearly different and so are the method they develop. In particular they do not consider the dynamical aspect of the transformation process and the link with the dynamic of the NNs.

\section{Conclusion}
We have shown that common approaches for UDT are ill-posed, and that in order to solve general UDT problems, information about the task must but be encoded in the problem setting. We propose doing this in a principled way, choosing an adequate representation space together with a transport cost $c$ encoding the semantic relation between elements of both domains in this space, and finding the OT between both domains \textit{w.r.t.} $c$. We then made use of the dynamical formulation of OT to introduce a general method for solving practically with NNs the OT-UDT problem as a constrained optimization on the trajectories of data points between the two domains. This formulation allows us to inject and control the semantics of the translation task via the definition of adequate transport costs. Finally we have illustrated the behavior of our model on some typical examples.


\bibliographystyle{unsrt}
\bibliography{main}

\clearpage
\newpage
\newpage
\appendix

\section{Implementation Details}
\label{supp:details}
\vspace{-0.1cm}
We provide below a general description of the algorithm corresponding to our method. Given two unpaired datasets, one first pre-trains an encoder-decoder and uses the obtained encoding as the new data representation to be used for training our model. This learned representation is more suitable than initial raw data (e;g. images) for representing the semantics and defining relevant OT cost functions. Using the transformed dataset, one then proceeds to training using a mini-batch gradient procedure. First the forward equation is solved, which corresponds to a classical  forward pass through the model. The loss in eq. \ref{eq:discrete_unconstrained_opt} is then computed and a gradient step is performed on all the model parameters. The Lagrangian coefficients are then updated as indicated in section \ref{sec:algorithm} in order to satisfy the constraints of the optimization problem when training ends.

\begin{algorithm}[ht!]
\caption{Training Procedure}
\begin{algorithmic}
\linespread{1.35}\selectfont
\label{algo}
  \STATE {\bfseries Input:} Dataset of unpaired images $(I_\sA$, $I_\sB)$, sampled from $(\am$, $\bm)$, \\
  Initial coefficient $\lambda_0$, decay parameter $d$, initial parameters $\theta$
  \STATE Pretrain Encoder $E$ and decoder $D$
  \STATE Make dataset of encodings $(x = E(I_\sA)$, $y = E(I_\sB))$
  \FOR{$i = 1, \dots, M$}
    \STATE Randomly sample a mini-batch of $x$, $y$ 
    \STATE Solve forward equation $\path^x_{k+1} = \path^x_{k} + \Delta t \, v^{\theta_k} \!(\path^x_{k})$ , starting from $\path^x_{0} = x$
    \STATE Estimate loss $\mathcal{L} = \mathcal{C}_d(\theta) + \frac{1}{\lambda_i}D\!\left((\path^\bdot_1)_\sharp\am, \, \bm\right)$ on mini-batch
    \STATE Compute gradient $\frac{\dm \mathcal{L}}{\dm \theta}$ backpropagating through forward equation
    \STATE Update $\theta$ in the steepest descent direction
    \STATE $\lambda_{i+1} \gets \max( \lambda_i - d, 0)$
  \ENDFOR
\STATE {\bfseries Output:} Learned parameters $\theta$.
\end{algorithmic}
\end{algorithm}
\vspace{-0.2cm}

\paragraph{Architectures.}
Implementation is performed via DCGAN and ResNet architectures as described below.
For the Encoder, we use a standard DCGAN architecture\footnote{\url{https://github.com/pytorch/examples/tree/master/dcgan}}, augmenting it with $2$ self-attention layers, mapping the images to a fixed, $128$ dimensional latent vector. For the Decoder, we use residual up-convolutions, also augmented with $2$ self-attention layers.

For the transportation, we use residual blocks very similar to those in the Resnet architecture proposed in CycleGan\footnote{\url{https://github.com/junyanz/pytorch-CycleGAN-and-pix2pix}}: we have 9 residual blocks, each consisting of a linear layer, batch normalization, a non-linearity, and a final linear layer.

The discrepancy $D$ is implemented using generative adversarial networks, although we have observed interesting results with other metrics, \textit{e.g.} using Sinkhorn Distances \cite{cuturi} or MMD \cite{gretton}. For the discriminator, we have used a simple MLP architecture of depth $3$, consisting of linear layers with spectral normalization, and LeakyReLU($p=0.2)$. 

\paragraph{Hyperameters.} We have considered latent dimensions of size $128$, the initial coefficient $\lambda_0 = 1$, and the decay factor is set depending on the number of total iterations $M$, so as to be zero on the final iteration. Throughout all the experiments, we use the Adam optimizer with $\beta_1=0.5$ and $\beta_2=0.999$.

\paragraph{MNIST Digit Swap Task}Architectures and hyperparameters are the ones presented above. Our dataset is made of $64 \times 64$ images. We have placed resized $32 \times 32$ MNIST $0$ and $1$ digits, either on the left or on the right, depending on the domain. For training, each training domain is made of $5000$ digits of each class. For the test, we use all the $0$ and $1$ digits available in MNIST.

\paragraph{Celeba Male to Female Translation.}
Architectures and hyperparameters are the ones presented above. Our dataset is the CelebA dataset, resizing images to $128 \times 128$ pixels, without any additional transformation.

\clearpage
\newpage

\section{On of the Effect of Initialization Gain on UDT and Transport}
\label{supp:bias}

We have seen that standard approaches for UDT yield numerous (bad) solutions in theory, although this seems not to be a problem in practice. These methods must yield an inductive bias, effectively restricting the space of solutions. In this Section, using simple tasks (that nevertheless reflect usual UDT) , we attempt to characterize this bias. We find that:

\begin{itemize}
    \item[--] Accuracy of these methods is highly correlated with initialization gain of network parameters: as gain is increased, accuracy decreases.
    \item[--] Large initialization gain leads to high transportation \textit{w.r.t.} the Euclidean cost,
\end{itemize}

From this, it is possible to make the link between accuracy of the mapping and transportation cost.

\begin{figure}[htb!]
\centering
\includegraphics[width=.9\textwidth]{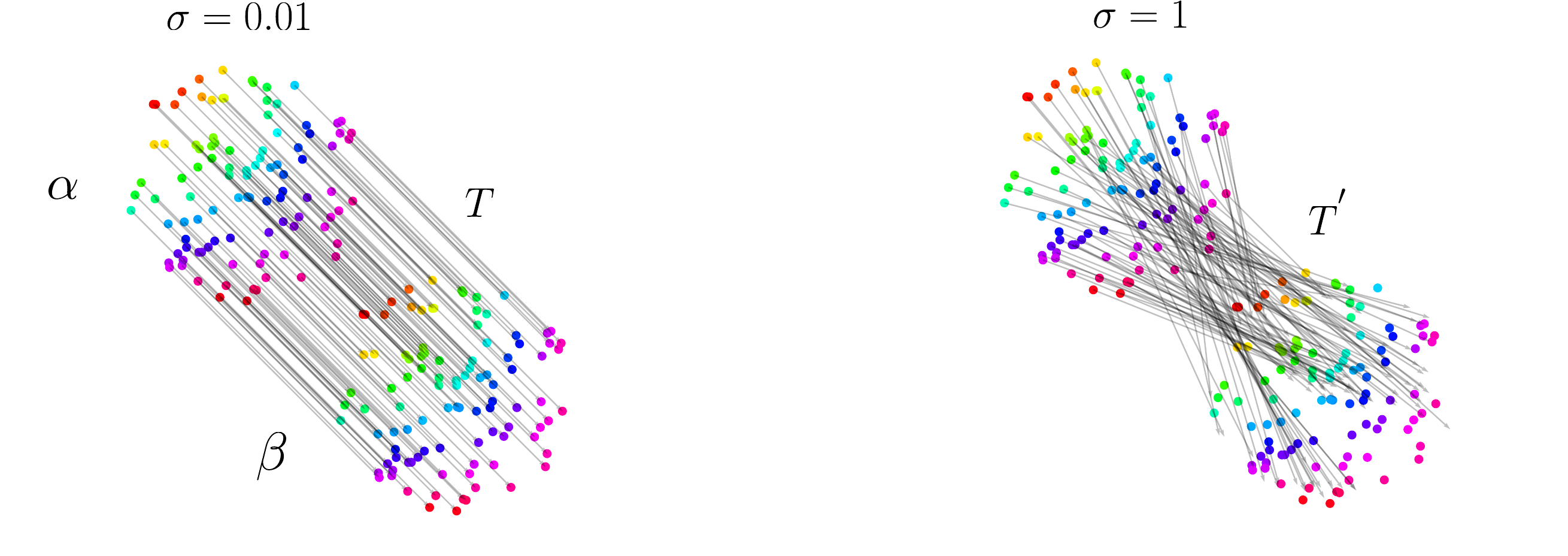}
    \caption{Large initialization ($\sigma=1$) leads to low quality pairing.}
    \label{fig:toy_transport}
\end{figure}

Consider the Toy Task in Fig. \ref{fig:toy_transport}: samples from a simple distribution $\am$ have been generated, and a simple transformation is applied to them to produce the target domain $\bm$. The goal is to retrieve the pairing between samples (\textit{i.e.} link samples with the same color). We have observed that keeping the original hyperparameters (small $\sigma$), CycleGan retrieves the accurate pairing (Fig. \ref{fig:toy_transport}, left). However, when the initialization gain of the mapping's parameters is augmented from $\sigma=0.01$ to $\sigma=1$, we no longer retrieve the accurate pairing ((Fig. \ref{fig:toy_transport}, right).

\begin{figure}[htb!]
\centering
\includegraphics[width=1.\textwidth]{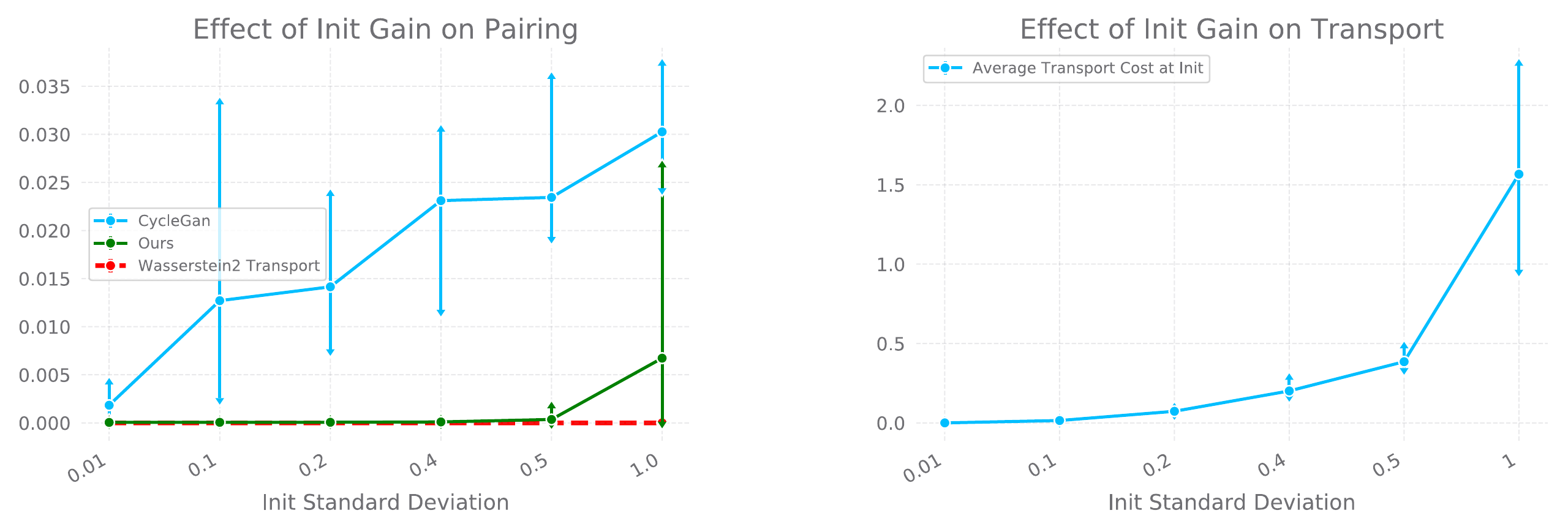}
    \caption{Influence of Initialization Gain on recovered pairings, and transport cost.}
    \label{fig:toy_transport2}
\end{figure}



This effect is quantified in Fig. \ref{fig:toy_transport2}. For the above example, the left curve represents the evolution of the pairing loss (Euclidean distance) between inputs and targets as initialization gain augments. Not only does the average distance between output and target sample augment, but so does the variance (results are obtained on 5 different runs for each hyperparameter). The right curve illustrates the evolution of the transportation \textit{w.r.t.} the Euclidean cost.  We observe that for high initialization gains, the transport cost at initialization is high and here too subject to an important variance. This highlights the link between between accurate pairings and transportation cost. Note that we have also tested our approach explicitly minimizing the transportation, which yields accurate results even for large  initialization gains (Fig. \ref{fig:toy_transport2})

\begin{figure}[bht!]
    \centering
    \includegraphics[width=1\textwidth]{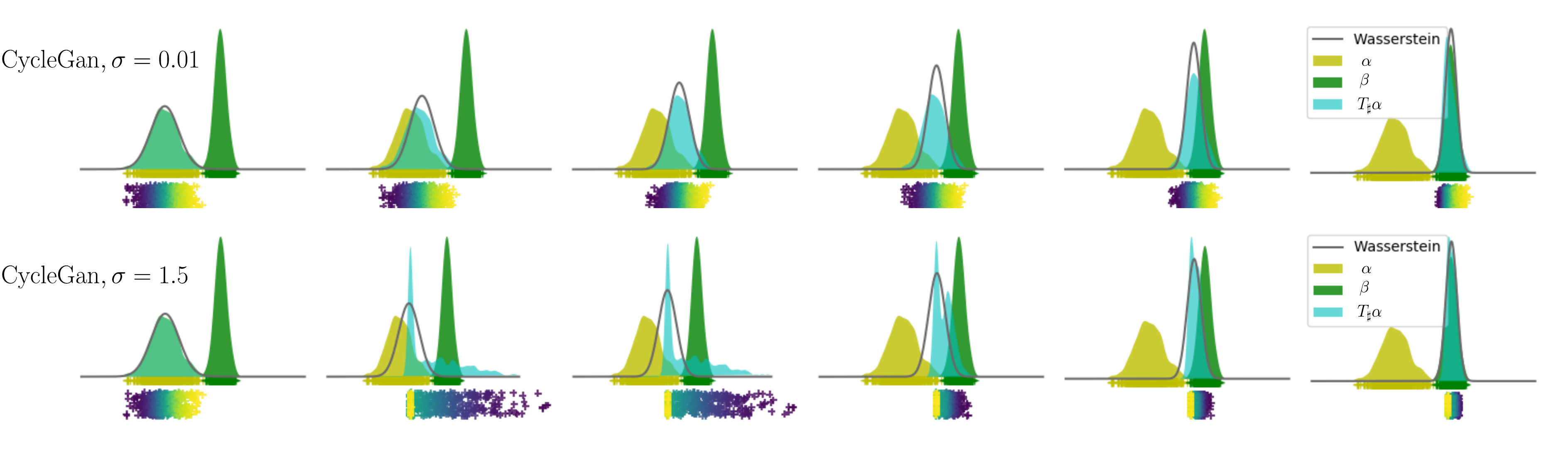}
    \caption{Hidden dynamics of CycleGan, for a simple task of mapping Gaussian $\am$ (leftmost yellow density on each plot) to $\bm$ (rightmost green density on each plot), for two different initialization gains: $\sigma=0.01$ and $\sigma=1.5$. A residual architecture of 5 blocks is trained on this UDT task.The evolution of the densities across the 5 layers are shown (2nd to 5th plot on each row) (light blue density). We also plot the optimal displacement of the density, \textit{w.r.t.} Wassertein2 (blue outlined curve on each plot).}
    \label{fig:toy_1d_dynamic}
\end{figure}

In a second series of experiments, one illustrates the inner dynamics of the residual architecture used for a 1d toy task, mapping a Gaussian distribution to another with different means and variances (Fig. \ref{fig:toy_1d_dynamic}). The task is described in the legend of the figure. For a small initialization (top row), we observe an interesting similarity with the optimal dynamic transport (using $c(x, y)=\|x-y\|^2$). For a large initialization, the inner dynamics are chaotic and the pairing is lost.

\section{Additional Theory}
\subsection{A Proof of the Ill-Posedness of the CycleGan method}
\label{supp:wellpos}

Here we prove the ill-posedness of the problem, in the case where domains $\am$ and $\bm$ are continuous, by constructing an infinite family of solutions:
\vspace{0.1cm}

\begin{prop}
For any two absolutely continuous distributions $\am$ and $\bm$, there exists an infinite number of transformations $T$ and $S$ satisfying the inversibility and coherence constraints. 
\end{prop}
\begin{proof}
From \cite{gaussianization} and the hypothesis that $\am$ and $\bm$ have densities w.r.t. the Lebesgue measure, we can take $G_{\am}$ and $G_{\bm}$ be transformations that gaussianize measures $\am$ and $\bm$ respectively. Let $R_\theta$ be a zero-centered  rotation of an arbitrary rotation angle $\theta$. It is clear that this transformation is measure-preserving for the Gaussian measure. Then, for any $\theta$, $T$ such that $T=G_{\bm}^{-1} \circ R_\theta \circ G_{\am}$ and its inverse $S$ verify the listed properties.
\end{proof}

\subsection{\label{app:dynot}From the Monge problem to Dynamical Optimal Transport}

Instead of directly pushing $\am$ to $\bm$ in $\R^d$, it is possible to view $\am$ and $\bm$ as points in a space of measures, and consider trajectories from $\am$ to $\bm$ in this abstract space. Thus, a way to transport the probability mass from $\am$ to $\bm$ is a curve between two points in this space. The curve corresponding to the optimal mapping is then the \textit{shortest} one, in other words it is the \textit{geodesic curve} between the two points.

More formally, let us introduce the \textit{Wasserstein metric space} $\mathbb{W}_p(\R^d)$, \textit{i.e.} the space of absolutely continuous measures of $\R^d$ with finite $p$-th moment endowed with the Wasserstein distance:
\[
W_p(\mu,\nu) = \min_{T_\sharp \mu=\nu} \C(T)^{\frac{1}{p}}
\]
when costs of the form $c(x,y) = \|x-y\|^p_p$ are considered, for some integer $p\geq2$. As $\mathbb{W}_p(\R^d)$ is a space of measures, $\am$ and $\bm$ are seen as points of this space of measures, and thus, any continuous path linking both distributions defines a gradual transformation from $\am$ to $\bm$ and a mapping transporting $\am$ to $\bm$.

The following result~(from Theorem 5.27 of \cite{Santambrogio}) motivates the dynamical formulation of OT:
\begin{prop}
$\mathbb{W}_p$ is a geodesic space, meaning that, for any measures $\mu, \nu\in\mathbb{W}_p$, there exists a geodesic curve $(\mu_t)_{ t \in [0, 1]}$ between $\mu$ and $\nu$.
\end{prop}

Thus, according to this result, finding the optimal mapping between two distributions amounts to finding a curve of minimal length in a certain abstract measure space. However, it still does not provide much in the way of a practically useful algorithm. The following theorem makes a formal link with fluid dynamics and basically states that moving probability masses from one distribution to another is the same as moving fluid densities from one configuration to another under a certain velocity field~\cite{Santambrogio}:
\begin{theorem} \label{thm:dynamic}
Given $\am$ and $\bm$ absolutely continuous \textit{w.r.t.} the Lebesgue measure and $(\mu_t)_{t \in [0, 1]}$ the geodesic curve with $\mu_0 = \am$ and $\mu_1 = \bm$, we can associate a vector field $v_t\in L^p(\mu_t)$ that solves the \textit{continuity equation}\footnote{$\partial_t$ is the partial derivative operator \textit{w.r.t.} variable $t$, and $\nabla \cdot$ the divergence operator \textit{w.r.t.} space.}:
\[
\partial_t\mu_t + \nabla\cdot(\mu_t v_t) = 0
\]
with:
\[
W^p_p(\am,\bm) = \int_0^1\|v_t\|^p_{L^p(\mu_t)}\mathrm{dt}
\]
\end{theorem}

In other words, the geodesic curve $(\mu_t)_{t \in [0, 1]}$ between both distributions, together with the minimal energy velocity vector field $v$ solve the continuity equation. Moreover, its energy along this path is precisely equal to the Wasserstein distance $W^p_p(\am, \bm)$. If this vector field of minimal energy $v$ could be obtained, probability mass could be displaced according to the flow defined by the continuity equation, and the geodesic curve could be retrieved. Thus, we can reformulate the problem as a problem of optimal control, where $v$ is the control variate:
\begin{mini}
    {v}{ \Cdyn(v) = \int_0^1 \norm{v_t}^p_{L^p(\mu_t)}\dm t}
    {}{}
    \addConstraint{\partial_t \mu_t + \nabla \cdot (\mu_t v_t)}{= 0, \mu_{0}= \am, \mu_{1}= \bm}
    \label{eq:dyn_supp}
\end{mini}

It is worth noting that this approach not only gives a mapping between the two distributions but it also gives the entire geodesic curve so that smooth interpolations in $\mathbb{W}_p(\R^d)$ can be recovered.

\newpage
\section{Additional Samples}
\label{supp:samples}
\begin{figure}[htb!]
    \centering
    \includegraphics[width=1\textwidth]{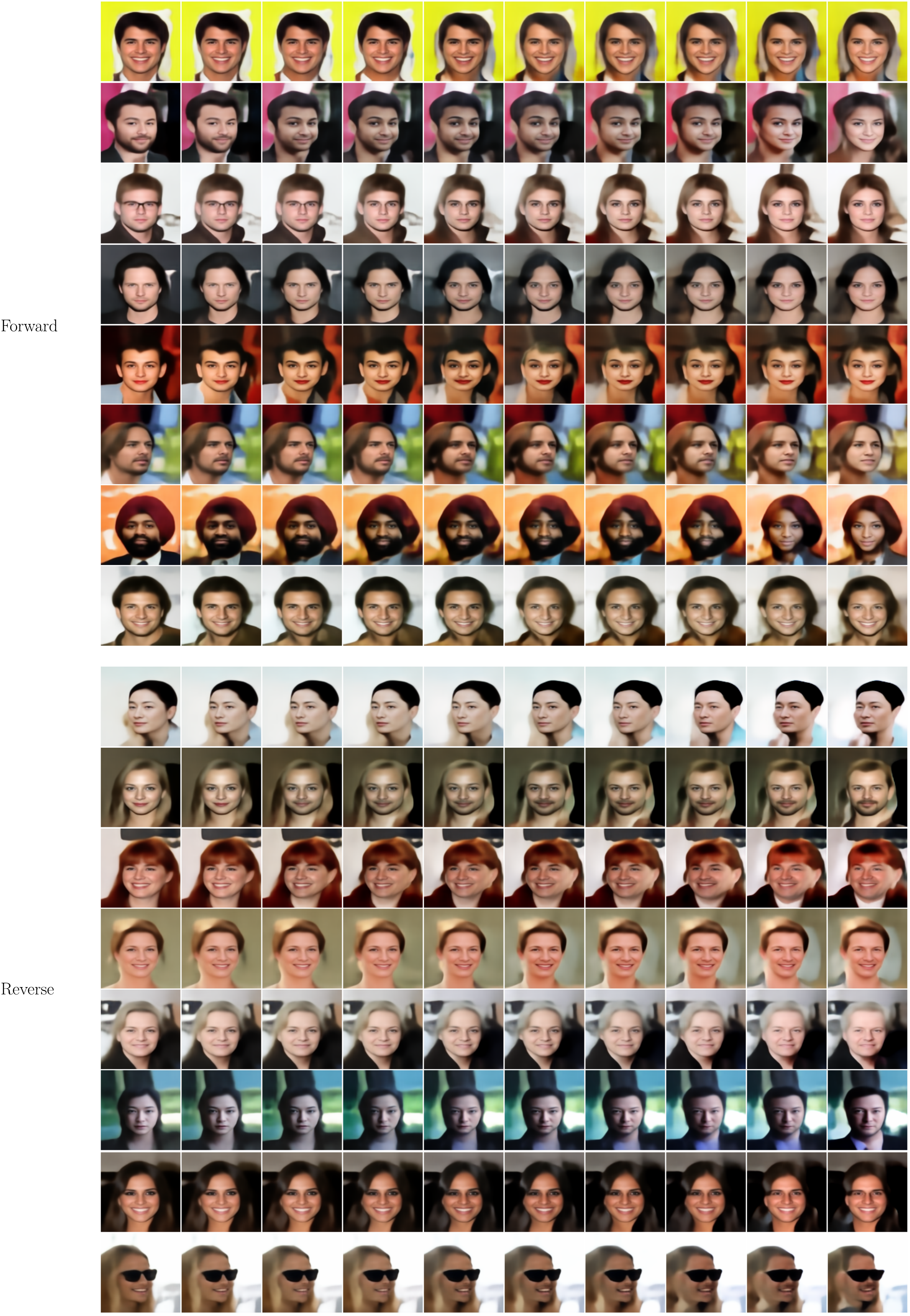}
    \caption{Male to Female, and Back.}
    \label{fig:h2f}
\end{figure}

\end{document}